\UseRawInputEncoding
\documentclass[12pt]{article}

\usepackage{dsfont}
\usepackage{amsfonts,amsmath,amsthm}
\usepackage{algorithm, algorithmic}
\usepackage{color}
\usepackage{graphicx}
\usepackage{stmaryrd}        

\usepackage[paper=a4paper,dvips,top=2cm,left=2cm,right=2cm,
    foot=1cm,bottom=4cm]{geometry}

\begin{document}
\large

\title{Augmented Quaternion and Augmented Unit Quaternion Optimization}
\author{ Liqun Qi\footnote{Department of Mathematics, School of Science, Hangzhou Dianzi University, Hangzhou 310018 China; Department of Applied Mathematics, The Hong Kong Polytechnic University, Hung Hom, Kowloon, Hong Kong
({\tt maqilq@polyu.edu.hk}).}
    \and \
    Xiangke Wang\thanks{College of Mechatronics and Automation, National University of Defence Technology, Changsha, 410073, China ({\tt
xkwang@nudt.edu.cn}).}
    \and \
     Chunfeng Cui\footnote{LMIB of the Ministry of Education, School of Mathematical Sciences, Beihang University, Beijing 100191 China.
    ({\tt chungfengcui@buaa.edu.cn}).}
}
\date{\today}
\maketitle

\begin{abstract}
In this paper, we introduce and explore augmented quaternions and augmented unit quaternions, and present an augmented unit quaternion optimization model.  An augmented quaternion consist of a quaternion and a translation vector.  The multiplication rule of augmented quaternion is defined.  An augmented unit quaternion consists of a unit quaternion and a translation vector.  The augmented unit quaternions form a Lie group.  By means of augmented unit quaternions, we study the error model and kinematics.   Then we formulate two classical problems in robot research, i.e., the hand-eye calibration problem and the simultaneous localization and mapping (SLAM) problem as augmented unit quaternion  optimization problems, which are actually real smooth spherical equality constrained optimization problems.  Comparing with the corresponding unit dual quaternion optimization model, the augmented unit quaternion optimization model has less variables and removes the orthogonality constraints.

\medskip


  \textbf{Key words.} Augmented quaternion, augmented unit quaternion, augmented unit quaternion optimization, hand-eye calibration, simultaneous localization and mapping.

\end{abstract}

\renewcommand{\Re}{\mathds{R}}
\newcommand{\rank}{\mathrm{rank}}
\renewcommand{\span}{\mathrm{span}}
\newcommand{\X}{\mathcal{X}}
\newcommand{\A}{\mathcal{A}}
\newcommand{\I}{\mathcal{I}}
\newcommand{\B}{\mathcal{B}}
\newcommand{\C}{\mathcal{C}}
\newcommand{\OO}{\mathcal{O}}
\newcommand{\e}{\mathbf{e}}
\newcommand{\0}{\mathbf{0}}
\newcommand{\dd}{\mathbf{d}}
\newcommand{\ii}{\mathbf{i}}
\newcommand{\jj}{\mathbf{j}}
\newcommand{\kk}{\mathbf{k}}
\newcommand{\va}{\mathbf{a}}
\newcommand{\vb}{\mathbf{b}}
\newcommand{\vc}{\mathbf{c}}
\newcommand{\vq}{\mathbf{q}}
\newcommand{\vg}{\mathbf{g}}
\newcommand{\pr}{\vec{r}}
\newcommand{\ps}{\vec{s}}
\newcommand{\pt}{\vec{t}}
\newcommand{\pu}{\vec{u}}
\newcommand{\pv}{\vec{v}}
\newcommand{\pw}{\vec{w}}
\newcommand{\pp}{\vec{p}}
\newcommand{\pq}{\vec{q}}
\newcommand{\pl}{\vec{l}}
\newcommand{\vt}{\rm{vec}}
\newcommand{\vx}{\mathbf{x}}
\newcommand{\vy}{\mathbf{y}}
\newcommand{\vu}{\mathbf{u}}
\newcommand{\vv}{\mathbf{v}}
\newcommand{\y}{\mathbf{y}}
\newcommand{\vz}{\mathbf{z}}
\newcommand{\T}{\top}

\newtheorem{Thm}{Theorem}[section]
\newtheorem{Def}[Thm]{Definition}
\newtheorem{Ass}[Thm]{Assumption}
\newtheorem{Lem}[Thm]{Lemma}
\newtheorem{Prop}[Thm]{Proposition}
\newtheorem{Cor}[Thm]{Corollary}
\newtheorem{example}[Thm]{Example}
\newtheorem{remark}[Thm]{Remark}

\section{Introduction}

In \cite{Qi22}, two classical problems in robot research, i.e., the hand-eye calibration problem \cite{Da99, LWW10, LLDL18, SA89, ZRS94} and
the simultaneous localization and mapping (SLAM) problem \cite{BS07, BLH19, CCCLSNRL16, CTDD15, WJZ13, WND00} were formulated as unit dual quaternion optimization problems.   In this optimization model, the variables are unit dual quaternions.   To convert to real optimization problems, each dual quaternion variable is equivalent to an eight-dimensional real vector, with one spherical equality constraint and one orthogonality constraint \cite{CLQY22}.   Both the spherical and orthogonality constraints are quadratic.
The optimization problem is smooth.   Then in \cite{Qi22a}, beside discussing some unit dual quaternion optimization models, a motion optimization model was presented.  In that model, each unit dual quaternion variable was replaced by a motion variable, which is equivalent to a six-dimensional real vector.  The spherical and orthogonality constraints were removed.  The equality constrained optimization problem turns to an unconstrained optimization problem.   This reduces the size of the problem.  However, further study shows that the motion operations are nonsmooth with some discontinuity points. See Appendix. This causes troubles.  Unless this discontinuity issue is overcome, this approach cannot work well.

This motivates us to consider a midway approach.   We propose to use an augmented unit quaternion variable to substitute the unit dual quaternion variable and the motion variable.  An augmented unit quaternion consists of a unit quaternion and a translation vector.  Thus, it is equivalent to a seven-dimensional real vector.
The spherical constraint is kept, and the orthogonality constraint is removed.   Our study reveals that the augmented unit quaternion operation is smooth.  Thus, the new approach keeps the smoothness and reduces the size of the problem simultaneously.  This new approach is attractive.

The distribution of the remainder of this paper is as follows.  In the next section, we review some basic properties of quaternions and dual quaternions.   Then we introduce augmented quaternions and augmented dual quaternions in Section 3.   An augmented quaternion has a quaternion part and a translation part.  Hence, it is equivalent to a seven-dimensional vector.   If the quaternion is a unit quaternion, then the augmented quaternion is called an augmented unit quaternion.  This needs an additional spherical constraint.  We define addition and multiplication for augmented quaternions.  We show that an augmented quaternion is invertible if and only if its quaternion part is invertible, and give the inverse formula.  In particular, when the involved augmented quaternions are augmented unit quaternions, the multiplication and inverse operations are smooth.  We show that the movement of a rigid body in the 3D space can be represented by an augmented unit quaternion.  We also show that the set of the augmented unit quaternions form a Lie group under its multiplication.    In Section 4, by means of augmented unit quaternions,
we study the motion error model and kinematics.  Then, in Section 5, we formulate the hand-eye calibration problem and the SLAM problem as augmented unit quaternion optimization problems.    The optimization problems are smooth, with some spherical constraints.   Some final remarks are made in Section 6.   An appendix on motion optimization is attached.

\section{Quaternions and Dual Quaternions}

The field of real numbers is denoted as $\Re$.  
The space of three-dimensional real vectors is denoted as $\mathbb V$.  We use small letters with overline arrows, such as $\pp, \pq$, to denote three-dimensional vectors.  For $\pp \in \mathbb V$, we may denote that $\pp = [p_1, p_2, p_3]$, but we still regard $\pp$ as a column vector and its transpose $\pp^\top$ a row vector. In particular, denote $\vec{0} := [0, 0, 0]$.

A {\bf quaternion} $\tilde q = [q_0, q_1, q_2, q_3]$ is a real four-dimensional vector.   We use a tilde symbol to distinguish a quaternion.   We may also write $\tilde q = [q_0, \pq ]$, where $\pq = [q_1, q_2, q_3] \in \mathbb V$.  See \cite{CKJC16, Da99, WHYZ12}.
Denote the set of all quaternions by $\mathbb Q$.   Suppose that we have two quaternions
$\tilde p = [p_0, \pp], \tilde q = [q_0, \pq] \in \mathbb Q$, where $p = [p_1, p_2, p_3], q = [q_1, q_2, q_3] \in \mathbb V$.  The sum of $\tilde p$ and $\tilde q$ is defined as
$$\tilde p + \tilde q = [p_0+q_0, \pp+\pq].$$
Denote $\tilde 0 := [0, 0, 0, 0] \in \mathbb Q$ as the zero element of $\mathbb Q$.
The product of $\tilde p$ and $\tilde q$ is defined by
$$\tilde p\tilde q = [p_0q_0-\pp \cdot \pq, p_0\pq+q_0\pp+\pp \times \pq],$$
where $\pp \cdot \pq$ is the dot product, i.e., inner product of $\pp$ and $\pq$, with
$$\pp \cdot \pq \equiv \pp^\top \pq = p_1q_1+p_2q_2+p_3q_3,$$
and $\pp \times \pq$ is the cross product of $\pp$ and $\pq$, with
$$\pp \times \pq = [p_2q_3-p_3q_2,-p_1q_3+p_3q_1, p_1q_2-p_2q_1] = - \pq \times \pp.$$
Thus, in general, $\tilde p \tilde q \not = \tilde q \tilde p$, and we have $\tilde p\tilde q = \tilde q\tilde p$ if and only if $\pp \times \pq = \vec{0}$, i.e., either $\pp = \vec{0}$ or $\pq = \vec{0}$, or $\pp = \alpha \pq$ for some real number $\alpha$.   Though the quaternion multiplication is not commutative, but it is associative, i.e., for any $\tilde p, \tilde q, \tilde s \in \mathbb Q$, we have
$$(\tilde p\tilde q)\tilde s = \tilde p(\tilde q\tilde s).$$
For convenience, denote
\begin{equation} \label{cross}
T(\pq)=\left(\begin{aligned} 0 \ \ & q_3 & -q_2 \\
-q_3 \ \ & \ \ 0 & q_1\\ q_2 \ \ & \ \ -q_1 & 0 \end{aligned}\right).
\end{equation}
Then we have $\pp\times \pq = T(\pq)\pp = T(\pp)^\top \pq$.

The conjugate of a quaternion $\tilde q = [q_0, q_1, q_2, q_3] \in \mathbb Q$ is defined as $\tilde q^* = [q_0, -q_1, -q_2, -q_3]$.    Let $\tilde 1 := [1, 0, 0, 0] \in \mathbb Q$.  Then for any $\tilde q \in \mathbb Q$, we have $\tilde q\tilde 1 = \tilde 1\tilde q = \tilde q$, i.e., $\tilde 1$ is the idenity element of $\mathbb Q$.     For any $\tilde p, \tilde q \in \mathbb Q$, we have
$$(\tilde p\tilde q)^* = \tilde q^*\tilde p^*.$$

Suppose that $\tilde p, \tilde q \in \mathbb Q$, $\tilde p \tilde q = \tilde q \tilde p = \tilde 1$.  Then we say that $\tilde p$ is invertible and its inverse is $\tilde p^{-1} = \tilde q$.

Suppose that $\tilde q = [q_0, q_1, q_2, q_3] \in \mathbb Q$.  Then its magnitude is defined as
$$|\tilde q| := \sqrt{q_0^2+q_1^2+q_2^2 + q_3^2}.$$
And $\tilde q$ is invertible if and only if $|\tilde q|$ is positive.  In this case, we have
\begin{equation} \label{eq:Qinverse}
\tilde q^{-1} = {\tilde q^* \over |\tilde q|^2}.
\end{equation}

A quaternion $\tilde q \in \mathbb Q$ is called a {\bf unit quaternion} if $|\tilde q|=1$.
Denote the set of all unit quaternions by $\mathbb U$.   

If $\tilde p, \tilde q \in \mathbb U$, then $\tilde p\tilde q \in \mathbb U$.  For any $\tilde q \in \mathbb U$, we have
$$\tilde q\tilde q^* = \tilde q^* \tilde q = \tilde 1,$$
i.e., $\tilde q$ is invertible and $\tilde q^{-1} = \tilde q^*$.

It is well-known that a unit quaternion can represent the rotation of a rigid body \cite{Ku99}.

A quaternion $\tilde q = [0, q_1, q_2, q_3] \in \mathbb Q$ is called a {\bf vector quaternion}.  A quaternion $\tilde q \in \mathbb Q$ is a vector quaternion if and only if $\tilde q = - \tilde q^*$.
Suppose that $\tilde q$ is a vector quaternion and $\tilde p$ is a quaternion, then $\tilde p^*\tilde q\tilde p$ is still a vector quaternion.

If a unit quaternion $\tilde q \in \mathbb U$ represents the rotation of a rigid body in the space, then we may write $\tilde q = [\cos \theta, \pl\sin \theta]$, where $0 \le \theta < 2\pi$ is the rotation angle of the rigid body, $\pl$ is the rotation axis, which is a unit vector in $\mathbb V$.   The logarithm of $\tilde q = [\cos \theta, \pl\sin \theta] \in \mathbb U$ is a vector quaternion: $\ln \tilde q := [0, \theta \pl]$.

A {\bf dual quaternion} $\hat q = [\tilde q; \tilde q_d]$ is a real eight-dimensional vector.  It consists of two quaternions $\tilde q$, the standard part of $\hat q$, and $\tilde q
$, the dual part of $\hat q$.  We use a hat symbol to distinguish a dual quaternion.  We denote the set of dual quaternions as $\hat {\mathbb Q}$.   Let $\hat p = [\tilde p; \tilde p_d], \hat q = [\tilde q; \tilde q_d] \in \hat {\mathbb Q}$.   Then the sum of $\hat p$ and $\hat q$ is
$$\hat p + \hat q = [\tilde p + \tilde q; \tilde p_d+\tilde q_d]$$
and the product of $\hat p$ and $\hat q$ is
$$\hat p\hat q = [\tilde p\tilde q; \tilde p\tilde q_d + \tilde p_d\tilde q].$$
Again, in general, $\hat p \hat q \not = \hat q\hat p$, but for any $\hat p, \hat q, \hat s \in \hat {\mathbb Q}$, we have
$$(\hat p\hat q)\hat s = \hat p (\hat q\hat s).$$
The conjugate of $\hat p = [\tilde p; \tilde p_d]$ is $\hat p^* = [\tilde p^*; \tilde p_d^*]$.

A dual quaternion $\hat p = [\tilde p; \tilde p_d]$ is called a {\bf unit dual quaternion} if $\tilde p$ is a unit quaternion and
\begin{equation} \label{eq:orthogonality}
\tilde q\tilde q_d^* + \tilde q_d\tilde q^* = 0.
\end{equation}
The set of unit dual quaternions is denoted as $\hat {\mathbb U}$.

A dual quaternion $\hat p = [\tilde p; \tilde p_d]$ is called a {\bf vector dual quaternion} if both $\tilde p$ and $\tilde p_d$ are vector quaternions.


\bigskip

\section{Augmented Quaternion}

An {\bf augmented quaternion} (AQ) $x = [\tilde p, \pt]$ is a real seven-dimensional vector.  Here, $\tilde p \in \mathbb Q$ is a quaternion and $\pt \in \mathbb V$ is a three-dimensional vector.  We call $\tilde p$
the {\bf quaternion part} of $x$, and $\pt$ the {\bf translation part} of $x$.  If $\tilde p$ is a unit quaternion, then $x$ is called an {\bf augmented unit quaternion} (AUQ).  Denote the set of AQs as $\mathbb A$, and the set of AUQs as $\mathbb {AU}$.

We may also denote $x$ as $x = [p_0, \pp, \pt]$, where $p_0$ is a scalar, i.e., a real number, and $\pp \in \mathbb V$.

Let $x = [\tilde p, \pt] = [p_0, \pp, \pt] \in \mathbb A$, $y = [\tilde q, \pu] = [q_0, \pq, \pu] \in \mathbb A$ and $\alpha \in \Re$.   Then we define the following operations:
\begin{equation}
x + y := [\tilde p + \tilde q, \pt + \pu],
\end{equation}
\begin{equation}
\alpha x := [\alpha \tilde p, \alpha \pt].
\end{equation}
Denote $0_A := [\tilde 0, \vec{0}]$.  Then $0_A$ is the zero element of $\mathbb A$.
These operations make $\mathbb A$ a vector space over reals.

We further define multiplication in $\mathbb A$ by
\begin{equation}  \label{eq:Aproduct}
x \circ y = [\tilde p\tilde q, \pu + R(\tilde q)^\top \pt],
\end{equation}
where
\begin{equation} \label{eq:R}
R(\tilde q)^\top = 2\pq\pq^\top + (q_0^2 - \pq^\top\pq)I -2q_0T(\pq)^\top.
\end{equation}
The addition and multiplication only involves polynomials, which are infinite times continuously differentiable.   For $R$, we have the following proposition.

\begin{Prop} \label{PropR}
For any $\tilde p, \tilde q \in \mathbb Q$, we have

(a) $R(\tilde p\tilde q) = R(\tilde p)R(\tilde q)$,

(b) $R(\tilde q^*) = R(\tilde q)^\top$. 
\end{Prop}
 \begin{proof}
     (a) For any $\vec{t}\in\mathbb V$ and $\tilde t=[0,\vec{t}]$, it holds that
     \begin{equation}\label{equ:R}
         [0, R(\tilde q)\pt]=\tilde q \tilde t\tilde q^*.
     \end{equation}
    Combining \eqref{equ:R} with $(\tilde p\tilde q) \tilde{t}(\tilde q^*\tilde p^*)=\tilde p(\tilde q \tilde{t}\tilde q^*)\tilde p^*$, this result is derived.

     (b)  It can be derived directly by    $T(-\vec{q})=T(\vec{q})^\top$ and \eqref{eq:R}.
 \end{proof}

Like quaternions, in general, $x \circ y \not = y \circ x$.   However, the multiplication is still associative.  We have the following proposition.

\begin{Prop} \label{P3.2}
For any $x, y, z \in \mathbb A$, we have
\begin{equation} \label{asso}
(x \circ y) \circ z = x \circ (y \circ z).
\end{equation}
\end{Prop}
\begin{proof}  Assume that $x = [\tilde p, \pt], y = [\tilde q, \pu]$ and $z = [\tilde s, \pv]$.  Then we have
\begin{eqnarray*}
(x \circ y) \circ z & = & [\tilde p\tilde q, \pu + R(\tilde q)^\top \pt] \circ [\tilde s, \pv]\\
& = & [(\tilde p\tilde q)\tilde s, \pv + R(\tilde s)^\top (\pu + R(\tilde q)^\top \pt)]\\
& = & [\tilde p(\tilde q\tilde s), \pv + R(\tilde s)^\top \pu + (R(\tilde q)R(\tilde s))^\top \pt].\\
\end{eqnarray*}
\begin{eqnarray*}
x \circ (y \circ z) & = & [\tilde p, \pt]\circ [\tilde q\tilde s, \pv + R(\tilde s)^\top \pu]\\
& = & [\tilde p(\tilde q\tilde s), \pv + R(\tilde s)^\top \pu + R(\tilde q\tilde s)^\top \pt].\\
\end{eqnarray*}
By Proposition \ref{PropR}, we have (\ref{asso}).
\end{proof}

Denote $e := [\tilde 1, \vec{0}]$.  Then $e \in \mathbb {AU} \subset \mathbb A$.  For any $x \in \mathbb A$, we have $x \circ e = e \circ x = x$.   Hence, $e$ is the identity element of $\mathbb A$.  

Suppose that $x, y \in \mathbb A$.  If $x \circ y = y \circ x = e$, then $x$ is called invertible, $y$ is called the inverse of $x$, and denoted as $x^{-1}$.

We have the following theorem.

\begin{Thm} \label{t3.1}
Suppose that $x = [\tilde p, \pt]  \in \mathbb A$.  Then 
$x$ is invertible if and only if $\tilde p$ is invertible.  In this case, we have the following formula:
\begin{equation}  \label{eq:Ainverse}
x^{-1} = \left[\tilde p^{-1}, -{R(\tilde p) \over |\tilde p|^4}\pt\right].
\end{equation}
\end{Thm}
\begin{proof}
Suppose that $x$ is invertible and has an inverse $y = [\tilde q, u]$.  Then by $x \circ y = y \circ x = e = [\tilde 1, \vec{0}]$ and (\ref{eq:Aproduct}), we have $\tilde p \tilde q = \tilde q \tilde p = \tilde 1$, i.e., $\tilde p$ is invertible.

On the other hand, assume that $\tilde p$ is invertible.  Then $|\tilde p| > 0$.  Let
$$y = \left[\tilde p^{-1}, -{R(\tilde p) \over |\tilde p|^4}\pt\right].$$
By (\ref{eq:Aproduct}) and Proposition \ref{PropR}, we may verify that $x \circ y = y \circ x = e$.  Then the theorem is proved.
\end{proof}

In particular, if $x = [\tilde p, \pt]$ is an AUQ, then $|\tilde p| = 1$.  The formula (\ref{eq:Ainverse}) is very simple then.   We have the following corollary.

\begin{Cor} \label{c3.2}
If $x = [\tilde p, \pt] \in \mathbb {AU}$, then it always has an inverse
\begin{equation}  \label{eq:AUinverse}
x^{-1} = \left[\tilde p^*, -R(\tilde p)\pt\right].
\end{equation}
\end{Cor}








Thus, the mapping $f(x) \equiv x^{-1} \in C^\infty$, i.e., infinite times continuously differentiable in this case.



For $x = [\tilde p, \pt] \in \mathbb {AU}$, we may use $\tilde p$ to represent the rotation of a rigid body, and $\pt$ to represent the translation of that rigid body.  Then $x$ represents the $3D$ movement of that rigid body.

We have the following theorem.

\begin{Thm}\label{thm:xoy3Dmeaning}
Suppose that $x, y \in \mathbb {AU}$.  Then $x \circ y$ represents the combined 3D movement of a rigid body, with $y$ succeeded by $x$.
\end{Thm}
\begin{proof}
Let $x = [\tilde p, \pt]$ and $y = [\tilde q, \pu]$.
We know that the motion of a rigid body in the 3D space can be represented by unit dual quaternions.   Then AQ $x$ is corresponding to a unit dual quaternion
\begin{equation}
\hat p = \left[\tilde p; {1 \over 2}\tilde p\tilde t\right],
\end{equation}
and
$y$ is corresponding to a unit dual quaternion
\begin{equation} \label{ee13}
\hat q = \left[\tilde q; {1 \over 2}\tilde q\tilde u\right].
\end{equation}
Since (\ref{eq:orthogonality}) holds directly for $\hat p$ and $\hat q$, combining $\hat p$ with $\hat q$, we have
\begin{eqnarray*}
\hat p \hat q & = & \left[\tilde p \tilde q; {1 \over 2} \tilde p(\tilde t\tilde q + \tilde q \tilde u)\right]\\
& = & \left[(\tilde p \tilde q); {1 \over 2}(\tilde p \tilde q)(\tilde q^* \tilde t \tilde q + \tilde u)\right].
\end{eqnarray*}
Further, $\tilde q^* \tilde t \tilde q$ is a vector quaternion and its vector value is equal to $R(\tilde q)^T\vec{t}$.
By this and (\ref{eq:Aproduct}), we see that $x \circ y$ is corresponding to $\hat p\hat q$.  The conclusion follows.
\end{proof}

For $x = [\tilde p, \pt]$, let $[\tilde 1,\vec{t}]$ represent a pure translation, and $[\tilde p,0]$ represent a pure rotation. Then we have
\begin{equation}
    x=[\tilde p,0]\circ[\tilde 1,\vec{t}].
\end{equation}
 In other words, $x$ represents that the rigid body first translates along direction $\vec{t}$ and then rotates along $\tilde p$.
 Similarly, we have
 \begin{equation}
    x^{-1}=[\tilde 1,-\vec{t}]\circ [\tilde p^*,0].
\end{equation}
 In other words, the inverse of $x$ represents that the rigid body first  rotates back along $\tilde p^*$ and then translates back along direction $-\vec{t}$.

If $y = [0, \pr, \pt] \in \mathbb V$, then $y$ is called an {\bf augmented vector quaternion} (AVQ).  Suppose that $y$ is an AVQ and $x$ is an invertible AQ, then we may show that $x \circ y \circ x^{-1}$ is still an AVQ.   All the AVQs form a six-dimensional space.  We denote it as $\mathbb {AV}$.

If $x = [\tilde p, \pt] \in \mathbb {AU}$ represents the movement of a rigid body in the 3D space, then we may write $x = \left[\cos {\theta \over 2}, \pl\sin {\theta \over 2}, \pt\right]$, where $0 \le \theta < 2\pi$ is the rotation angle of the rigid body, $\pl$ is the rotation axis, which is a unit vector in $\mathbb V$, and $\pt \in \mathbb V$ is the translation vector.   Similar to the logarithm of a unit quaternion, we may define the logarithm of an AUQ $x = \left[\cos {\theta \over 2}, \pl\sin {\theta \over 2}, \pt\right]$ as an AVQ: $\ln x := [0, {\theta \over 2}\pl, {1 \over 2}\pt]$.
Then $\ln x \in \mathbb {AV}$.

We have the following theorem on $\mathbb {AU}$.


\begin{Thm}\label{thm:Liegroup}
The AUQ set $\mathbb {AU}$ is a Lie group under the AQ multiplication.
\end{Thm}
\begin{proof}  In the last section, we know that $e$ is the identity element of $\mathbb {AU}$.  By Theorem
\ref{t3.1}, every element of $\mathbb {AU}$ has an inverse.  By (\ref{eq:Aproduct}), the multiplication $\circ$ is closed in $\mathbb {AU}$, i.e., the product of two AUQs is still an AUQ.  By Proposition \ref{P3.2}, the multiplication $\circ$ is associative.   Thus, $\mathbb {AU}$ is a group.

We now show that $\mathbb {AU}$ is a manifold.  We see that $\mathbb {AU} = \mathbb U \times \mathbb V$.  It is well-known that $\mathbb U$ is diffeomorphic to manifold $S^3$ \cite{MR94}.  Thus, $\mathbb U$ is a manifold with three dimensions.  On the other hand, $\mathbb V$ is the three-dimensional space, thus a manifold with three dimensions.  Hence, $\mathbb {AU}$ is a manifold with six dimensions.

Finally, for any $x, y \in \mathbb {AU}$, denote $F(x, y) = x \circ y^{-1}$.  By Corollary \ref{c3.2} and (\ref{eq:Aproduct}), $F(x, y)$ is $C^\infty$.

Therefore, $\mathbb {AU}$ is a Lie group under the AQ multiplication $\circ$.
\end{proof}

It is well-known that the tangent space of a Lie group at its identity element is a Lie algebra.
We will identify such a tangent space of $\mathbb {AU}$ in the next section.

\bigskip

\section{Error Model and Kinematics}

We now study a possible application of AUQ in kinematic control.
{Let $x=[\tilde p,\pt]\in \mathbb{AU}$ and $\tilde p=[\cos \frac{\theta}{2}, \sin\frac{\theta}{2}\pl]$, where $\theta$ denotes the rotation angle and $\pl$ denotes the  rotation axis. Then we have
	\begin{equation}
		\dot x=\left[\dot {\tilde p}, \dot \pt\right],
	\end{equation}
where
\begin{eqnarray*}
    \dot {\tilde p}
    &=&  \left[-\frac{\dot\theta}{2}\sin \frac{\theta}{2}, \frac{\dot\theta}{2}\cos\frac{\theta}{2}\pl\right]\\
    &=&\frac12\left[\cos \frac{\theta}{2}, \sin\frac{\theta}{2}\pl][0,\dot\theta \pl\right]\\
    &=&\frac12 \tilde p \tilde w.
\end{eqnarray*}
Here, $\vec{w}=\dot\theta \pl$	is the angular velocity and $\tilde w=[0,\vec{w}]$.

We denote $T_e(\mathbb{AU})$ as the tangent space of the Lie group $\mathbb {AU}$  at the identity $e=[\tilde 1,\vec{0}]$. For any $x\in\mathbb{AU}$ and the corresponding element  $T_x$ in $T_e(\mathbb{AU})$, we have
\begin{equation}
    T_x=\dot x|_{x=e}=\left[\frac12 \tilde p\tilde w,\dot{\vec{t}}\right]\bigg\vert_{x=e}=\left[\frac12 \tilde w,\dot{\vec{t}}\right].
\end{equation}
Therefore, $T_e(\mathbb{AU})$ is the space composing of all AVQ as $\left[\frac12 \tilde w,\dot{\vec{t}}\right]$.

Denote $x=[\tilde p,\pt]\in \mathbb{AU}$ as the current AUQ and $x_d=[\tilde p_d,\pt_d]\in \mathbb{AU}$ as the target AUQ. Then define the error   as
\begin{eqnarray}
  \nonumber x_e &=& x^{-1}\circ  x_d \\
  \nonumber &=& [\tilde p^*, - R(\tilde p)\pt]\circ[\tilde p_d,\pt_d]\\
  \nonumber  &=& [\tilde p^*\tilde{p}_d, -R(\tilde p_d)^\top R(\tilde p)\pt+\pt_d]\\
    &\equiv&[\tilde p_e,\vec{t}_e].   \label{equ:error}
\end{eqnarray}
In the last equality, $\tilde p_e\equiv\tilde p^*\tilde{p}_d$ and $\vec{t}_e\equiv-R(\tilde p_e)^\top \pt+\pt_d$. Here, by Proposition \ref{PropR}, we have $R(\tilde p_e)=R(\tilde p)^\top R(\tilde p_d)$.
 Equation (\ref{equ:error}) can be taken as the error model of a motion. Then we have
 \begin{eqnarray}
  \nonumber  \dot{\tilde p}_e &=&   (\dot{\tilde p})^*\tilde{p}_d + \tilde p^* \dot{\tilde{p}}_d \\
  \nonumber  &=& \frac12 [(\tilde p\tilde w)^*\tilde{p}_d + \tilde p^*(\tilde{p}_d\tilde w_d)]\\
   \nonumber &=&\frac12\tilde p_e (\tilde w_d-\tilde p_e^*\tilde w\tilde p_e)\\
    &\equiv \frac12 \tilde p_e\tilde w_e,\label{equ:dotpe}
 \end{eqnarray}
where $\tilde w_e=\tilde w_d-\tilde p_e^*\tilde w\tilde p_e$ is a vector quaternion.

 In the sequel, a kinematical error model is proposed.


\begin{Thm}
	Take $x=[\tilde p,\pt]\in \mathbb{AU}$ and $x_d=[\tilde p_d,\pt_d]\in \mathbb{AU}$ as the current AUQ and the target AUQ, respectively.
	Then
	the kinematical error  model  of \eqref{equ:error}
	has the following formulation
	\begin{equation}\label{equ:contol}
		\dot{x}_e = \frac12 (x_e\circ \xi_e),
	\end{equation}
	where
	\begin{equation}\label{equ:xie}
	 \xi_e=[\tilde w_e,\vec{v}_e] \text{  and  }\vec{v}_e=2\dot{\vec{t}}_e-2(\vec{w}_e^\top\vec{t}_e)\vec{w}_e+2(\vec{w}_e^\top\vec{w}_e)\vec{t}_e.
	\end{equation}
\end{Thm}
\begin{proof}
By \eqref{equ:dotpe}, we have $\dot{x}_e=[\dot{\tilde p}_e,\dot{\vec{t}}_e]=[\frac12 \tilde p_e\tilde w_e,\dot{\vec{t}}_e]$.
Let $\vec{v}_e\in\mathbb V$ such that
\begin{eqnarray*}
  [\frac12 \tilde p_e\tilde w_e,\dot{\vec{t}}_e]=\frac12\left([\tilde p_e,\vec{t}_e]\circ[\tilde w_e,\vec{v}_e]\right).
\end{eqnarray*}
Then we have
\begin{eqnarray*}
	\vec{v}_e=2\dot{\vec{t}}_e-R(\tilde w_e)^\top\vec{t}_e.
\end{eqnarray*}
Combining this with $(\tilde w_e)_0=0$ and \eqref{equ:R}, equation \eqref{equ:xie} is derived. This completes the proof.
\end{proof}

The control input of the kinematical error model  \eqref{equ:contol} and \eqref{equ:xie} is $\xi_e$.
Suppose that $\xi_d$ and $x_e$ are known in prior, then the actual control input is $\xi$, which is consists of angular velocity $\vec{w}$ and linear velocity $\dot{\vec{t}}$.
 Note that we have proved that AUQ is a Lie-group in Theorem \ref{thm:Liegroup}. By using this information, the generalized proportional control law is proposed as follows.
\begin{Thm}
	 For the kinematical error model  \eqref{equ:contol} and \eqref{equ:xie}, the generalized proportional control law
	 \begin{equation}\label{equ:controllaw}
	 	\xi_e=-2[0,K_r \vec{\theta}_e,K_t\vec{t}_e],
	 \end{equation}
 exponentially stabilizes configuration $x$ to configuration $x_d$ globally, where $[0,\vec{\theta}_e]=\ln \tilde q_e$, $K_r=\text{diag}(\vec{k}_r)$ and $K_t=\text{diag}(\vec{k}_t)$ are positive definite matrices.
\end{Thm}
\begin{proof}
	By direct computations, we have
	 \begin{equation}\label{equ:thetadot}
	 	(\vec{\theta}_e)^\top \vec{w}_e=(\vec{\theta}_e)^\top\dot{\vec{\theta}}_e.
	 \end{equation}
 	Substituting \eqref{equ:xie} into \eqref{equ:controllaw}, we have
 	\begin{eqnarray}
 		&&\vec{w}_e=-2K_r \vec{\theta}_e, \label{equ:we}\\
 		&&\dot{\vec{t}}_e=-K_t\vec{t}_e+(\vec{w}_e^\top\vec{t}_e)\vec{w}_e-(\vec{w}_e^\top\vec{w}_e)\vec{t}_e.\label{equ:dotve}
 	\end{eqnarray}
 Consider the Lyapunov function candidate $V$
 \begin{equation}\label{equ:Ve}
 V_e=\alpha \|\vec{\theta}_e\|^2+\beta\|\vec{t}_e\|^2.
 \end{equation}
Then $V_e$ is positive definite, and when $\|\vec{\theta}_e\|\rightarrow \infty$ and $\|\vec{t}_e\|\rightarrow \infty$, we have $V_e\rightarrow \infty$.

Taking the derivative of \eqref{equ:Ve} and using \eqref{equ:thetadot}, \eqref{equ:we}, and \eqref{equ:dotve}, we have
\begin{eqnarray*}
	\dot{V}_e&=&2\alpha \vec{\theta}_e^\top\dot{\vec \theta}_e + 2\beta \vec{t}_e^\top\dot{\vec{t}}_e\\
	&=& 2\alpha \vec{\theta}_e^\top(-2K_r \vec{\theta}_e) + 2\beta \vec{t}_e^\top(-K_t\vec{t}_e+(\vec{w}_e^\top\vec{t}_e)\vec{w}_e-(\vec{w}_e^\top\vec{w}_e)\vec{t}_e)\\
	&\le&  -4\alpha \vec{\theta}_e^\top K_r \vec{\theta}_e - 2\beta \vec{t}_e^\top K_t\vec{t}_e,
\end{eqnarray*}
where the last inequality follows from the Cauchy inequality.

Consequently, let $k_{\min}=\min(\vec{k_r},\vec{k}_t)$, we have
\begin{equation}
	\dot{V}_e \le -2k_{\min}(\alpha \|\vec{\theta}_e\|^2+\beta\|\vec{t}_e\|^2) =-2k_{\min}V_e.
\end{equation}
Thus, control law \eqref{equ:controllaw} guarantees  configuration $x$  exponentially converging to configuration $x_d$ globally with converging rate $e^{-2k_{\min}}$. The proof is completed.
\end{proof}
}

\bigskip

\section{Augmented Unit Quaternion Optimization}

For an AQ $x = [\tilde p, \pt] = [p_0, p_1, p_2, p_3, t_1, t_2, t_3] \in \mathbb A$, define its magnitude as
\begin{equation}
|x| = \sqrt{p_0^2 + p_1^2 + p_2^2 + p_3^2 + \sigma t_1^2 + \sigma t_2^2 + \sigma t_3^2},
\end{equation}
where $\sigma$ is a positive number.   

We define an AQ vector $\vx = [x^{(1)}, \cdots, x^{(n)}]$ as an $n$-component vector such that for $i = 1, \cdots, n$, its $i$th components is an AQ $x^{(i)} = \left[\tilde p^{(i)}, \pt^{(i)}\right] = \left[p_0^{(i)}, p_1^{(i)}, p_2^{(i)}, p_3^{(i)}, t_1^{(i)}, t_2^{(i)}, t_3^{(i)}\right] \in \mathbb A$.  Thus, $\vx$ may be also regarded as a $7n$-dimensional real vector.  If $x^{(i)} \in \mathbb {AU}$ for $i = 1, \cdots, n$, then we say that $\vx$ is an AUQ vector.  We use small bold letters such as $\vx$ to denote AQ vectors and AUQ vectors, and denote the space of $n$-component AQ vectors by ${\mathbb A}^n$.  The norm in ${\mathbb A}^n$ is defined by
\begin{eqnarray}
&& \|\vx\| \nonumber \\
& = & \sqrt{\sum_{i=1}^n \left|x^{(i)}\right|^2} \nonumber \\
& = & \sqrt{\sum_{i=1}^n \left|p_0^{(i)}\right|^2 + \left|p_1^{(i)}\right|^2 + \left|p_2^{(i)}\right|^2 + \left|p_3^{(i)}\right|^2 + \sigma\left|t_1^{(i)}\right|^2 + \sigma\left|t_2^{(i)}\right|^2 + \sigma\left|t_3^{(i)}\right|^2}.
\end{eqnarray}
Then it is a norm.   Also denote the space of $n$-component AUQ vectors by ${\mathbb {AU}}^n$.

 Assume that  $\vz : {\mathbb {AU}}^n \to {\mathbb A}^m$. An {\bf augmented unit quaternion optimization problem} is formulated as
\begin{equation}
\min \left\{ {1 \over 2}\| \vz(\vx)\|^2 : \vx \in {\mathbb {AU}}^n \right\},
\end{equation}
which is a $7n$-dimensional equality constrained optimization problem, with $7n$ real variables, and $n$ spherical equality constraints:
\begin{equation}
\left(p_0^{(i)}\right)^2 + \left(p_1^{(i)}\right)^2 + \left(p_2^{(i)}\right)^2 + \left(p_3^{(i)}\right)^2 = 1,
\end{equation}
for $i = 1, \cdots, n$.


\medskip

 {\bf Example 5.1}  We now consider the 1989 Shiu and Ahmad \cite{SA89} and Tsai and Lenz \cite{TL89} hand-eye calibration model.  Then $n = 1$.   We may formulate it as AUQ equations
 \begin{equation} \label{eq12}
 a^{(i)} \circ x =  x \circ b^{(i)},
 \end{equation}
 where, $a^{(i)}, b^{(i)}, x \in {\mathbb {AU}}$, for $i = 1, \cdots, m$.    Here, AUQ $x$ reprsents the transformation from the camera (eye) to the gripper (hand), $a^{(i)}, b^{(i)}$, for $i = 1, \cdots, m$, are some data AUQs from experiments.   The aim is to find the best AUQ $x$ to satisfy (\ref{eq12}).  Then, let the $i$th error be
 \begin{equation} \label{eq16}
 z^{(i)} =  \left(a^{(i)} \circ x\right) - \left(x \circ b^{(i)}\right),
 \end{equation}
 and $\vz = [z^{(1)}, \cdots, z^{(m)}] \in {\mathbb A}^m$.
We have the following augmented unit quaternion optimization problem
\begin{equation} \label{eq16}
\min \left\{ {1 \over 2}\| \vz(x)\|^2 : x \in \mathbb {AU} \right\}
\end{equation}
for this hand-eye calibration model.   This is a $7$-dimensional optimization problem with one spherical equality constraint.   {\bf Most importantly, by the discussion in Section 3, this optimization problem is a smooth optimization problem}.

\medskip

  {\bf Example 5.2}  We further consider the 1994 Zhuang, Roth and Sudhaker \cite{ZRS94} hand-eye calibration model.  Then $n = 2$.   We may formulate it as AUQ equations
 \begin{equation} \label{eq18}
 a^{(i)} \circ x =  y \circ b^{(i)},
 \end{equation}
 where, $a^{(i)}, b^{(i)}, x, y \in {\mathbb {AU}}$, for $i = 1, \cdots, m$.    Here, $y$ is the transformation AUQ from the world coordinate system to the robot base.  The aim is to find the best AUQs $x$ and $y$ to satisfy (\ref{eq18}).  Then, let
 \begin{equation}
 \vx = \left[x, y\right],
 \end{equation}
 \begin{equation}
 z^{(i)} = \left(a^{(i)} \circ x \right) - \left( y \circ  b^{(i)}\right),
 \end{equation}
 and $\vz = [z^{(1)}, \cdots, z^{(m)}] \in {\mathbb A}^m$.
We have the following augmented unit quaternion optimization problem
\begin{equation}  \label{eq20}
\min \left\{ {1 \over 2}\| \vz(\vx)\|^2 : \vx \in {\mathbb {AU}}^2 \right\}
\end{equation}
for this hand-eye calibration model.   This is a $14$-dimensional optimization problem with two spherical equality constraints.   {\bf Similar to the last example, this optimization problem is a smooth optimization problem}.

\medskip

  {\bf Example 5.3} We now consider the simultaneous localization and mapping (SLAM) problem.   We have a directed graph $G = (V, E)$ \cite{CTDD15}, where each vertex $i \in V$ corresponds to a robot pose $x^{(i)} \in {{\mathbb {AU}}}$ for $i = 1, \cdots, n$, and each directed edge (arc) $(i, j) \in E$ corresponds to a relative measurement $y^{(ij)}$, also an AUQ.  There are $m$ directed edges in $E$.   The aim is to find the best $x^{(i)}$ for $i = 1, \cdots, n$, to satisfy
  \begin{equation}
   y^{(ij)} = \left( x^{(i)}\right)^{-1} \circ x^{(j)}
  \end{equation}
  for $(i, j) \in E$.  We now formulate it as an augmented unit quaternion optimization problem.  Let
  $\vx = \left[x^{(1)}, \cdots, x^{(n)}\right]$.   Then $\vx \in {\mathbb {AU}}^n$. Let
   \begin{equation}
   z^{(ij)} = \left( x^{(i)}\right)^{-1} \circ  x^{(j)} - y^{(ij)}
   \end{equation}
  for $(i, j) \in E$ and $\vz = [z^{(ij)} : (i, j) \in E] \in {\mathbb A}^m$, where $m$ is the number of directed edges.   Then we have the following augmented unit quaternion optimization problem
\begin{equation} \label{eq25}
\min \left\{ {1 \over 2}\| \vz(\vx)\|^2 : \vx \in {\mathbb {AU}}^n \right\}
\end{equation}
for the SLAM problem.   This is a $7n$-dimensional optimization problem with $n$ spherical equality constraints.  {\bf Again, it is a smooth optimization problem}.

An augmented unit quaternion consists of a unit quaternion and a translation vector.  Thus, it is equivalent to a seven-dimensional real vector.  It needs a spherical constraint. A unit dual quaternion is equivalent to a eight-dimensional real vector.  It needs not only the spherical constraint, but also the orthogonality constraint (\ref{eq:orthogonality}).    Thus, comparing with the unit dual quaternion optimization model for the hand-eye calibration problem and the SLAM problem \cite{Qi22, Qi22a}, the augmented unit quaternion model here keeps the smoothness and reduces the size of the problem simultaneously.   On the other hand, the motion optimization model \cite{Qi22a} reduces the size of the problems further, but it has some discontinuity.  See Appendix.   Hence, the augmented unit quaternion model is better.


\section{Final Remarks}

In this paper, we introduced augmented quaternions and augmented unit quaternions, and established their basic theories.   We also formulated the hand-eye calibration problem and the SLAM problem as augmented unit quaternions optimization problem.   Comparing with the corresponding unit dual quaternion optimization problems, the augmented unit quaternion optimization problem reduces the size of the problem and removes the orthogonality constraint, yet the model is smooth, thus workable in practice.


In Section 4, we studied the error model and kinematics by augmented dual quaternions. We leave further exploration in this direction as a future research topic.   We also leave the issue for developing numerical algorithms for augmented unit quaternion optimization for our further research task.

\bigskip

\section{Appendix: Motion Optimization}

A motion $x$ is a six-dimensional real vector, which has two parts: $x = [\pr, \pt]$, where $\pr$ is a rotation, while $\pt$ is a translation.  Here, $\pr = [r_1, r_2, r_3] = \theta \pl$, where $\pl$ is a unit vector, representing the rotation axis, and  $\theta$ satisfying $0 \le \theta < 2\pi$, and representing the rotation angle.    Then we may establish relations between motions and unit dual quaternions \cite{Qi22a}.   The movement of a rigid body in the 3D space may also be represented by a motion vector.  As a unit dual quaternion is an eight-dimensional vector with a spherical constraint and an orthogonality constraint, the motion representation is the most brief representation of a rigid body movement.

A unit dual quaternion optimization problem can be converted to a motion optimization problem.   There are two ways to do this.   One way was presented in \cite{Qi22a}.  The motion variables are mapped to unit dual quaternions, to take operations, then mapped back to motion variables.   The key operator there is
the rotation operator \cite{Qi22a}
\begin{equation} \label{eq1.1}
R(\tilde q) = \left\{ \begin{aligned} {2 \cos^{-1} q_0 \over \sqrt{q_1^2+q_2^2+q_3^2}}[q_1, q_2, q_3], & \ {\rm if}\  q_0^2 \not = 1, \\
0, &  \ {\rm otherwise},
\end{aligned} \right.
\end{equation}
where $\tilde q = [q_0, q_1, q_2, q_3]$ is a unit quaternion.   This operator is discontinuous at $q_0 = -1$.    Another way is to define operations between motions to replace the multiplication of unit dual quaternions.   Suppose we have two motions $x = [\pr, \pt]$ and $y = [\ps, \pu]$.   Then we may define $x \odot y$ to replace the multiplication of the corresponding unit dual quaternions.   The key part of this operation is the operation between $\pr$ and $\ps$.  Denote this as $\pr \oplus \ps$.   Then we discover that this operation is discontinuous when $\pr \oplus \ps = \vec{0}$, but $\ps \not = \vec{0}$, ${\pr \over \|\pr\|_2} = {\ps \over \|\ps\|_2}$.  Hence, the motion optimization model cannot work well, unless this discontinuity issue is overcome.


\bigskip



\bigskip




\begin{thebibliography}{10}
\bibitem{Qi22} L. Qi, ``Standard dual quaternion optimization and its applications in hand-eye calibration and SLAM'', {\sl Communications on Applied Mathematics and Computation} DOI. 10.1007/s42967-022-00213-1.
\end{thebibliography}
\end{document}